\documentclass{article}

\usepackage{graphicx} 
\usepackage{amsmath,amsfonts,amssymb,amsthm}
\usepackage{xcolor}
\usepackage{algorithm2e}
\RestyleAlgo{ruled}
\usepackage{bbm}
\usepackage{tikz}
\usepackage{algorithmic}
\usepackage{mathtools}
\usepackage{makebox}
\usepackage{hyperref}
\usepackage{subcaption}
\usepackage{natbib}
\usepackage{authblk}



\usepackage{fullpage}

\usepackage{float}

\newtheorem{definition}{Definition}
\newtheorem{lemma}{Lemma}

\newtheorem{corollary}{Corollary}

\newcommand{\cX}{\mathcal{X}}
\newcommand{\cY}{\mathcal{Y}}
\newcommand{\cP}{\mathcal{P}}
\newcommand{\cD}{\mathcal{D}}

\newcommand{\cT}{\mathcal{T}}

\newcommand{\E}{\mathbb{E}}
\newcommand{\ecd}{\mathbb{E}_{\cD}}
\newcommand{\1}{\mathbbm{1}}
\newcommand{\calC}{\mathcal{C}}

\title{Model Ensembling for Constrained Optimization}
\author{Ira Globus-Harris}
\author{Varun Gupta}
\author{Michael Kearns}
\author{Aaron Roth}
\affil{University of Pennsylvania}

\begin{document}

\maketitle

\begin{abstract}
There is a long history in machine learning of model ensembling, beginning with boosting and bagging and continuing to the present day. Much of this history has focused on combining models for  classification and regression, but recently there is interest in more complex settings such as ensembling policies in reinforcement learning. Strong connections have also emerged between ensembling and multicalibration techniques. In this work, we further investigate these themes by considering a setting in which we wish to ensemble
models for multidimensional output predictions that are in turn used for downstream optimization. More precisely, we imagine we are given a number of models mapping a state space to multidimensional real-valued
predictions. These predictions form the coefficients of a linear objective that we would like to optimize under specified constraints. The fundamental question we address is how to improve and combine such models in a way that outperforms the best of them in the downstream optimization problem. We apply multicalibration techniques that lead to two provably efficient and convergent algorithms. The first of these (the \emph{white box} approach) requires being given models that map states to output predictions, while the second (the \emph{black box} approach) requires only \emph{policies} (mappings from states to solutions to the optimization problem). For both, we provide convergence and utility guarantees. We conclude by investigating the  performance and behavior of the two algorithms in a controlled experimental setting.
\end{abstract}
\section{Introduction}

Many instances of decision making under uncertainty can be decomposed into two steps: \emph{prediction} and \emph{optimization}. For example, when deciding on a portfolio of investment assets, we might first predict the returns of individual assets, then choose the portfolio that maximizes predicted return subject to  budget and risk constraints. Similarly, when deciding on a route to drive, we might first predict the congestion along each road segment, and then solve a shortest-path problem on the road network to minimize predicted travel time. The predictive component of such a decision making algorithm would take as input a context relevant to the task at hand (e.g. past returns, weather conditions, time of day, etc.) and would produce a vector-valued prediction (e.g. the return of each stock, or the congestion of each road). When paired with a corresponding optimization problem (e.g. maximizing returns subject to risk constraints or minimizing travel time) the predictive component induces a policy, mapping contexts to feasible actions. In practice, such optimization problems are applied to a wide range of settings, from healthcare and delivery services to resource scheduling and inventory stock allocation (e.g. \cite{gammelli2022predictive, stratigakos2022prescriptive, chu2023data, deo2015planning, donti2017task, gallien2015initial}). 

Now suppose we have multiple such predictive models that make (different) predictions, and produce (different) policies. In this paper we develop methods for \emph{ensembling} these systems to produce policies that can strictly improve on the constituent or base policies. This is similar in motivation to classical ensembling methods such as boosting and bagging (see e.g. \cite{schapire2013boosting}) and more recent work focused on ensembling policies in reinforcement learning (see e.g. \cite{cheng2020policy,liu2023active}), but in a much more complex action space in which contexts are mapped to high dimensional actions via constrained optimization. We give two such ensembling methods. The first operates in a \emph{white box} model, in which we have direct access to the predictive model, and not just to the policy it induces. The second operates in a \emph{black box} model, in which we have access only to the policies we are ensembling, and assume nothing about how they are derived ---e.g. they may or may not be the result of optimizing over predictions. For both methods, we give strong theoretical guarantees and evaluate their performance in a set of controlled numerical experiments.

\subsection{Our Results}
We study a setting in which the decision maker solves a $d$-dimensional optimization problem with an objective that is a linear function of the label variables $y$ that the prediction method aims to predict. The optimization problem can be defined by a set of specified (and not necessarily convex) constraints.  

Our white box ensembling method  follows a simple, intuitive idea. Given a context vector $x$, each method $h_i$ produces a \emph{predicted} label vector $h_i(x)$, which, after solving the corresponding optimization problem produces the action $\pi_i(x)$. If model $i$'s predictions were correct, then the corresponding payoff of the model's action would be $\pi_i(x) \cdot h_i(x)$. We call this model $i$'s self-assessed payoff. The idea is to always take the action of the model with the largest self-assessment: $\arg\max_i \pi_i(x) \cdot h_i(x)$.  But this idea runs into several obstacles. First, the self-assessed expected payoff of a model $\E[\pi_i(x)\cdot h_i(x)]$ need not have any clean relationship with its actual expected payoff $\E[\pi_i(x)\cdot y]$. Second, even if each model is ``self-consistent'' in the sense that $\E[\pi_i(x)\cdot y]=\E[\pi_i(x)\cdot h_i(x)]$, we would expect to lose self-consistency after selecting the model with the highest self-assessment, because we would be conditioning on a model having an unusually high self-assessed payoff: this would result in upward bias for the selected model \emph{conditional on the selection event} even if the models produced independent, unbiased forecasts of the outcome variables. We solve these problems by showing how to efficiently post-process the models so that they have consistent self-assessments even conditional on their selection---using techniques deriving from the multicalibration literature \cite{hebert2018multicalibration}. For each model $i$ in the set that we are ensembling, our conditions guarantee: 
$$\E\left[\pi_i(x)\cdot (h_i(x)-y) \middle| i = \arg\max_{i'}\pi_{i'}(x)\cdot h_{i'}(x)\right]=0.$$

We show that the ensembled policy $\pi^*$ that results from selecting the action $\pi^*(x)=\pi_{i^*(x)}(x)$ (where $i^*(x) = \arg\max_i[\pi_i(x)\cdot h_i(x)]$ is the index of the model with the highest self-assessment) is self-consistent and is guaranteed to obtain expected payoff that is at least (up to error terms):
$$\E\left[\pi^*(x)\cdot y \right] = \E\left[\pi^*(x)\cdot h_{i^*}(x) \right]\geq \E\left[\max_i\pi_i(x)\cdot h_i(x) \right].$$
Note that the maximum is inside the expectation, and so this is the \emph{point-wise} maximum self-assessed payoff. This can be substantially higher than the expected payoff of the best constituent model, which is (because of the self-consistency condition) $\max_i \E[\pi_i(x)\cdot y] = \max_i \E[\pi_i(x)\cdot h_i(x)]$.

Our white box ensembling method offers a strong guarantee, but has three  limitations. First, it requires access to the predictive model $h_i(x)$ that underlies the policy $\pi_i(x)$ This might not always be available --- and might not always even exist, as there are methods to train policies for action that do not involve explicit prediction of labels. Second, it requires updating and maintaining all of the constituent models to be ensembled, which might be prohibitive if the number of models is large. Thirdly, its reward guarantees hold  with respect to the models after they have been post-processed to satisfy our self-consistency conditions, rather than with respect to the original ensemble of (potentially non self-consistent) models. 

Towards addressing these limitations, we introduce an alternative ``black box" ensembling method. This method maintains only a single predictive model, and only requires access to the policies $\pi_i(x)$ to be ensembled, not any details of their implementation. In broad strokes, what it does is maintain a predictive model $h^*(x)$ for the labels that is unbiased conditional on each coordinate of its own induced action $\pi^*(x)$ as well as conditional on each coordinate of the actions $\pi_i(x)$ chosen by each of the models to be ensembled. We prove a swap regret-like guarantee for this ensemble: not only is it self-consistent, it obtains payoff higher than any of the constituent models even conditional on any coordinate of its induced action. Because this technique only requires maintaining a single predictive model and requires fewer evaluations of the downstream optimization problem, it can be trained substantially faster than our white box ensembling method. However, as a consequence of the limitations of its black box access to the constituent policies, it does not give the same form of strong point-wise guarantee that the white box approach does. 

We then conduct a set of numerical experiments to evaluate our two ensembling methods. We show that both methods outperform the initial models, and that the white box method often empirically converges to the optimal predictor. We see that the black box model tends to converge faster, and is substantially more efficient in practice. 

\subsection{Related Work}
The problem of finding policies that solve optimization problems in the face of unknown label vectors is often solved by first predicting the label vectors and then optimizing for the predicted label. The ``Smart Predict then Optimize'' framework of \cite{elmachtoub2022smart} focuses on the design of surrogate loss functions to minimize in the prediction training phase that are best suited for the particular downstream optimization task. In addition to the main difference that our work considers ensembling multiple models rather than training single predictors from scratch, we do not use a surrogate loss function that incorporates the downstream optimization, but instead leverage techniques from the multicalibration literature, which originates from algorithmic fairness concerns, to get guarantees on the downstream objective \cite{hebert2018multicalibration}. 

A line of work on omniprediction in both the unconstrained and constrained settings \cite{gopalan2022omnipredictors,garg2024oracle,hu2023omnipredictors,globus2023multicalibrated,gopalan2024omnipredictors} gives theorems that informally state that if a 1-dimensional predictor $f$ (usually for a binary outcome) is multicalibrated with respect to a set of models, then for some family of loss functions, $f$ has loss at most the loss of the best model in the class. Here the focus is generally on the ability of the model $f$ to perform well over multiple loss functions, and the promise is only that $f$ performs as well as the best model in the class, rather than strictly better than it. One notable exception is \cite{globus2023multicalibration} which analyzes multicalibration as a boosting algorithm for regression functions, and proves that if $f$ is multicalibrated with respect to a class of models $\mathcal{H}$, then it in fact performs as well as the best model in a strictly more expressive class (that $\mathcal{H}$ serves as weak learners for). Similarly, \cite{garg2019tracking} and \cite{roth2023reconciling} consider the problem of reconciling two 1-dimensional regression functions that have similar error, but make different predictions. They show how to combine two such models into a single, more accurate model. A primary point of departure for us is that we consider ensembling models for high dimensional optimization problems, rather than 1-dimensional classification and regression problems. 

The debiasing steps that we use are closest in spirit to those used in ``decision calibration'' \cite{zhao2021calibrating} or ``prediction for action" \cite{noarov2023highdimensional}, which aim to produce predictors that are unbiased conditional on the action taken by some downstream decision maker. Independently and concurrently of our work, \cite{stevenetal} adapt the ``reconciliation'' procedure of \cite{roth2023reconciling} to the decision calibration setting, updating pairs of models that frequently induce different decisions in downstream decision makers on the regions on which they induce different downstream decisions. Their end result is two new and reconciled models which agree with one another and are unbiased conditional on the action induced --- and their bounds inherit a polynomial dependence on $k$, the number of actions of the downstream decision maker. Because the optimization problems we consider have linear objectives, we only need that our predictions are unbiased subject to the \emph{coordinates} of the actions that result from optimization --- a fact that was also used by \cite{noarov2023highdimensional}. This is what lets us handle downstream optimization problems with very large action spaces. The focus of \cite{zhao2021calibrating,noarov2023highdimensional} was on producing policies that offer a downstream decision maker various forms of regret (like swap regret) --- in contrast, our interest is in ensembling multiple explicit policies. The focus of \cite{stevenetal} is on reconciling model multiplicity, whereas our focus is on achieving superior performance for the task at hand than the base models, and not explicitly on resolving disagreements between them.

\section{Preliminaries}

We assume there is a joint probability distribution $\cD$ over a \emph{context space} $\cX$ and a $d$-dimensional real-valued label space $\cY \subset \mathbb{R}^d$. Label vectors $y \in \cY$ are assumed to have bounded coordinates: $\|y\|_{\infty} \leq M$ for all $y \in \cY$.

There is an underlying optimization problem, to map contexts $x$ to $d$-dimensional actions $a \in \Omega \subseteq [0,1]^d$. Here $\Omega$ represents a known but arbitrary feasibility constraint --- there is, e.g., no requirement that $\Omega$ be convex. A policy is a mapping from contexts to feasible actions $\pi:\cX\rightarrow \Omega$. The payoff of an action $a \in \Omega$ given a label vector $y \in \cY$ is modeled as their inner product $a\cdot y$. Thus the expected payoff of a policy $\pi$ is $\E_{(x, y) \sim \cD}[\pi(x) \cdot y]$. 

If the labels $y$ were known at each round, then the optimal action to take would be $\arg\max_{a \in \Omega} a \cdot y$ --- the solution to an optimization problem with a linear objective and arbitrary constraints represented by $\Omega$. However, the labels $y$ are not known. One way to approach the decision-making under uncertainty problem is to first train a predictive model $h:\cX\rightarrow \cY$ that maps contexts to predicted label vectors. Such a predictive model $h$ induces a policy $\pi_h(x) = \arg\max_{a \in \Omega}a\cdot h(x)$, that finds the action that maximizes \emph{predicted} payoff given the constraints. A model $h$ induces a policy that has actual expected payoff $\E_{(x, y) \sim \cD}[\pi_h(x)\cdot y]$ --- but we will also be interested in a model's self-assessed or self-evaluated payoff. Given an example $x$, a model $h$'s self-assessed payoff is $\pi_h(x)\cdot h(x)$, and a model's expected self-assessment is $\E[\pi_h(x) \cdot h(x)]$. Absent further conditions, a model's self-assessed payoff need not have any relationship to its actual payoff. In the next section, we will define relevant conditions that we will impose on a model to relate its self-assessed payoff with its actual payoff. 
These conditions will make reference to a partition, or bucketing, of the unit interval $[0,1]$ into $\frac{1}{w}$ level sets each of width $w$ that we will refer to as $\cT$. We refer to specific level sets in $\cT$ as $\tau$ and the midpoint of an interval $\tau$ as $\tau_1.$ 

\subsection{Consistency Conditions}

In order to relate a model's self-assessed payoff to its actual payoff, we will leverage the ability to make conditionally ``consistent" predictions -- informally, predictions which are accurate on average, not just marginally, but also conditional on arbitrary sets of interest. The parameterization we choose is related to the multicalibration literature --- see e.g. \cite{roth2022uncertain} for a discussion of this and alternative parameterizations.

\begin{definition}[Consistent Predictions]
\label{def:consistency}
 Fix a distribution $\cD$. We say that the model $h:\cX\rightarrow \cY$ is $\alpha-$consistent with respect to a collection of sets $\mathcal{C} \subseteq 2^{\cX}$ if $\ \forall \ C \in \calC,$ it is the case that 
\begin{align*}
\|\ecd[y - h(x) | x \in C]\|_{\infty} \leq \frac{\alpha}{\Pr[x \in C]}.
\end{align*}
\end{definition}

  One important class of sets we will be concerned with making consistent predictions on is the level sets of different policies we aim to ensemble. 
  
\begin{definition}[Policy Level Sets]
\label{def:policy-level-sets}
    Fix a policy $\pi$ and bucketing $\cT$. We refer to the level sets of a policy $\pi$ as
    \[ \calC^{\cT}_{\pi} = \left\{ \{ x \mid \pi(x)_i \in \tau \}: {i \in [d], \tau \in \cT} \right\}, \]
    which is the collection of subsets of $\cX$ on which $\pi$ induces an action that allocates an amount in the interval $\tau$ to outcome coordinate $i$. 
\end{definition}

\begin{definition}[Consistency to a Policy]
\label{def:consistent-to-policy}
    Fix a model $h: \cX \to \cY$, policy $\pi: \cX \to \Omega$, and bucketing $\cT$. We say that the model $h$ is $\alpha-$consistent with respect to the policy $\pi$ if it is $\alpha-$consistent with respect to the level sets of $\pi$. 
    That is, 
    if for all
    $i \in [d], \tau \in \cT,$ it is the case that
        \begin{align*}
            \|\ecd[y - h(x) | \pi_h(x)_i \in \tau ]\|_{\infty} \leq \frac{\alpha}{\Pr[\pi_h(x)_i \in \tau]}.
        \end{align*}
\end{definition}

As we will see, we will be able to post-process models $h$ to be consistent with collections of sets $\mathcal{C}$ that may be defined in terms of $h$ itself. In particular, it will be useful for us to ask that a model $h$'s predictions are consistent with respect to the policy $\pi_h$ that it itself induces. This will turn out to be the condition we need to make a model ``self-consistent'' the sense that its self-assessed payoff accords with its actual payoff. Thus when a model $h$ is consistent (in the sense of  Definition \ref{def:consistency}) with $\pi_h$, we will say that $h$ is \emph{self-consistent}.

\section{Consistent Predictions} \label{sec:consistent}

In this section, we describe the procedure for making consistent predictions and the transparent outcome guarantees that consistent models provide. The basic algorithm driving our approach is an iterative de-biasing procedure similar to the template that has become common in the multi-calibration literature \cite{hebert2018multicalibration}. At a high level it iteratively identifies subsets of the data domain on which the current model exhibits (statistical) bias. It then shifts the model's predictions on these identified regions to remove the statistical bias. Each shift decreases squared error on the underlying distribution, which is what guarantees the algorithm's fast convergence. 
Where we will differ from the multicalibration literature will be in our choice of bias events -- these will turn out to be ``cross-calibration'' events defined across multiple models, and defined in terms of the solution to the optimization problem induced by the prediction of our models. For simplicity we describe the algorithm as if it has access to the distibution $\cD$ directly, but out-of-sample guarantees follow straightforwardly from standard techniques --- see e.g. \cite{roth2022uncertain}.

\subsection{Making Consistent Predictions}

\begin{algorithm} 
\caption{Update (hyperparameters: consistency tolerance $\alpha$)}
\label{alg}
    \begin{algorithmic}
        \STATE {\bf Input} model $h^0$, collection of sets $\calC \subseteq 2^{\cX}$
        \STATE Initialize $t = 0$
        \WHILE{ $ \exists \ C \in \calC $ such that $\Pr[x \in C] \Vert\E_{\cD}[y - h^t(x) | x \in C]\Vert_{\infty} > \alpha$ }
            \item $t = t + 1$
            \item $\vec{\delta} \coloneqq\E_{\cD}[y - h^{t-1}(x) | x \in C] $
            \item $h^t(x) \coloneqq h^{t-1}(x) + \mathbbm{1}_{[x \in C]} \cdot \vec{\delta}$
        \ENDWHILE
        \STATE {\bf Output} model $h^t$
    \end{algorithmic}
\end{algorithm}

\subsubsection{Convergence of \textsc{Update} Procedure}

What drives the convergence analysis of this and similar algorithms is the fact that correcting statistical bias on a subset of the data domain is guaranteed to decrease the squared error of a model. Thus squared error can act as a potential function, even when the subsets on which we update the model intersect. The following is a standard lemma, first used by \cite{hebert2018multicalibration} in the multicalibration literature --- we adopt a variant used in \cite{roth2022uncertain}.

\begin{lemma}[Monotone Decrease of Squared Error. (\cite{roth2022uncertain}) ] \label{lem:monotone-sq-error}
    Fix a model $h$, distribution $\cD$, policy $\pi$, and set $C \subseteq \cX$. Let
    \begin{align*}
        \Delta = \ecd[y | x \in C ] &- \ecd[h(x) | x \in C ] \\
         \text{and} \\
        h'(x) &= \begin{cases}
            h(x) + \Delta & \text{if } x \in C , \\
            h(x) & \text{ o.w.}
        \end{cases}
    \end{align*}
    Then,
    \begin{align*}
        \ecd[\|h(x) - y\|_2^2] - \ecd[\|h'(x) - y\|_2^2] = \Pr[x \in C] \cdot  \| \Delta \|_2^2.
    \end{align*}
\end{lemma}
The convergence of the algorithm then follows from a potential argument.
\begin{lemma} \label{lem:update-conv}
    The procedure \textsc{Update}($h, \calC$) (Algorithm \ref{alg}) terminates within $dM^2/\alpha^2$ rounds. 
\end{lemma}
The proof of Lemma \ref{lem:update-conv} is deferred to Appendix \ref{ap:proofs}.

\subsection{Using Consistent Predictions}

In this section, we show that consistency with respect to carefully constructed sets of events allows us to evaluate the payoff of a policy induced by a model via its \emph{self-assessments}, and lets us compare the policy induced by a model with other policies. These statements will be the basic building blocks of our ensembling methods.

First we show that if a model $h$ is consistent with respect to a policy $\pi$'s level sets (Definition \ref{def:policy-level-sets}), then the model $h$'s predicted label can be used in place of the true label $y$ to correctly estimate the payoff of $\pi$. 

\begin{lemma} \label{lem:consistent}
Fix a distribution $\cD$ and a bucketing $\mathcal{T}$, with $w = \sqrt{\alpha / M}$. Let $\pi:\cX\rightarrow \Omega$ be an arbitrary policy. 
 Let $h:\cX\rightarrow \cY$ be a model that is $\alpha-$consistent with respect to $\pi$ (Definition \ref{def:consistent-to-policy}). Then:
$|\E_{\cD}[\pi(x) \cdot h(x)] - \E_{\cD} [\pi(x) \cdot y]| \leq 2d \sqrt{\alpha M} $. 
\end{lemma}

\begin{proof} [Proof of Lemma \ref{lem:consistent}]
First, we show that $\E_{\cD}[\pi(x) \cdot y] \geq \E_{\cD}[\pi(x) \cdot h(x)] - 2d \sqrt{\alpha M}$.

     \begin{align*}
        \ecd[\pi(x) \cdot y] &= \ecd[\sum_{i \in [d]} \pi(x)_i \cdot y_i ] \\
        &\geq \sum_{\tau \in \cT} \sum_{i \in [d]}  \Pr[\pi(x)_i \in \tau] \ecd[ \tau_1 \cdot y_i - |\tau_1 - \pi(x)_i| \cdot |y_i| \ | \pi(x)_i \in \tau] \\
        &\geq \sum_{\tau \in \cT} \sum_{i \in [d]}  \Pr[\pi(x)_i \in \tau] \left( \ecd[\tau_1 \cdot y_i | \pi(x)_i \in \tau] - \frac{wM}{2} \right)\\
        &\geq \sum_{\tau \in \cT} \sum_{i \in [d]} \Pr[\pi(x)_i \in \tau] \left( \ecd[\tau_1 \cdot h(x)_i | \pi(x)_i \in \tau] -\frac{\alpha}{\Pr[\pi(x)_i \in \tau]} -  \frac{wM}{2} \right) \\
        &\geq \sum_{\tau \in \cT} \sum_{i \in [d]} \Pr[\pi(x)_i \in \tau] \left( \ecd[ \pi(x)_i \cdot h(x)_i - |\tau_1 - \pi(x)_i| \cdot |h(x)_i| \ | \pi(x)_i \in \tau] -\frac{\alpha}{\Pr[\pi(x)_i \in \tau]} -  \frac{wM}{2} \right) \\
        &\geq \sum_{\tau \in \cT} \sum_{i \in [d]} \Pr[\pi(x)_i \in \tau] \left( \ecd[\pi(x)_i \cdot h(x)_i | \pi(x)_i \in \tau] -\frac{\alpha}{\Pr[\pi(x)_i \in \tau]} -  {wM} \right) \\
        &= \ecd[\pi(x) \cdot h(x)] - \alpha |\cT| d  - {wMd}, \\
        &= \ecd[\pi(x) \cdot h(x)] - \frac{\alpha d}{w}  - {wMd},
    \end{align*} 
where the first and fourth inequalities hold by the triangle inequality, the second and fifth by the assumption that $\|y\|_{\infty} \leq M$ for all $y \in \cY$, and the third inequality follows from $h$ satisfying consistency condition on policy $\pi$.
Setting $w = \sqrt{\alpha/M},$ we have $\ecd[\pi(x) \cdot y] \geq \ecd[\pi(x) \cdot h(x)] - 2d\sqrt{\alpha M}.$
The reverse direction holds similarly.
\end{proof}

An especially useful special case of Lemma \ref{lem:consistent} is the case in which a model $h$ is consistent with the policy $\pi_h$ that it itself induces. In this case, the model can be used to correctly evaluate its own payoff, and corresponds to a natural and useful ``transparency'' condition similar in spirit to the transparency conditions of \cite{zhao2021calibrating,noarov2023highdimensional}.

\begin{corollary} \label{cor:self-rev}
    Fix a distribution $\cD$ and a bucketing $\mathcal{T}$, with $w = \sqrt{\alpha / M}$.
    If a model $h: \cX \to \cY$ is $\alpha-$self-consistent, then its expected outcome is close to its expected self-evaluation:
    \[ |\E_{\cD} [\pi_h(x) \cdot h(x)] - \E_{\cD}[\pi_h(x) \cdot y]| \leq 2d \sqrt{\alpha M}. \]
\end{corollary} 

\paragraph{Comparing Policies}
It is also possible for a predictor to satisfy the consistency conditions with respect to \emph{multiple} policies. Why might this be useful? Informally, since you can trust a model's evaluation of any policy that it is consistent with respect to, if a model is consistent with many policies, optimizing according to its predictions should induce outcomes that are only better than those from the best policy it is consistent with. The next lemma shows that this is in fact the case. 

\begin{lemma} \label{lem:single-consistent}
Fix a distribution $\cD$ and a bucketing $\cT$, with $w = \sqrt{\alpha/M}$. 
Let $\pi: \cX \to \Omega$ be an arbitrary policy. Let $h: \cX \to \cY$ be an $\alpha-$self-consistent model that is also $\alpha-$consistent with respect to the policy $\pi$. 
Then: $\E_{\cD}[\pi_h(x) \cdot y] \geq \ecd[\pi(x) \cdot y] - 4 d \sqrt{\alpha M} $.
\end{lemma} 
The proof of Lemma \ref{lem:single-consistent} is deferred to Appendix \ref{ap:proofs}.

\section{White Box Ensembling Method}
\label{sec:multiple-models}
This section describes the first of two ensembling methods: a ``white box" method for ensembling $k$ models. Like the method we will describe in Section \ref{sec:single-model}, this method enjoys transparent outcome guarantees. 
However, this method requires strong access to the models being ensembled---access to their point predictions, rather than just the policies induced by the model. We prove that the final ensembled policy strictly improves on the reward of the constituent models after debiasing, by obtaining their \emph{pointwise maximum} self-assessed reward. The debiasing procedure is guaranteed to improve the squared error of the constituent predictive models, but not necessarily the reward of the policies they induce.
 In Section \ref{sec:experimental}, we verify that empirically the ensemble policy does \textit{strictly} improves on the payoff of the initial constituent policies. 

\paragraph{Interaction Model}
A decision maker has access to $k$ constituent models makings predictions of the coefficients of a linear objective function (e.g. stock prices), which they are using to make decisions subject to some arbitrary constraints. 
They build an ensemble using these $k$ models by updating them to satisfy consistency conditions which we describe formally below. In this scheme, the decision maker needs access not only to the policies they are incorporating into their decision making procedure, but the predictive models used to induce these policies, as the ensembling procedure involves iteratively modifying the predictions of each of these constituent models.

\begin{definition}[White Box Ensemble]
    A white box ensemble ${\bf h} = h_1, \ldots, h_k$ is a collection of $k$ models.
\end{definition}

The white box ensemble policy that the decision maker will employ is simple: selecting the constituent model that has the highest self-assessed payoff (or lowest, if the downstream optimization is a minimization). 
For simplicity, we will refer to the index $i^*(x)$ as the model in a white box ensemble that has the highest self-assessed payoff on a given point $x$: $i^*(x) = \arg\max_{j \in [k]} \pi_{h_j}(x) \cdot h_j(x)$.

\begin{definition}[White Box Ensemble Policy] \label{def:ensemble-policy}
    A white box ensemble policy $\pi_{\bf h}$ for a white box ensemble ${\bf h} = (h_1, \ldots, h_k)$ of models is the policy that, for each $x \in \cX$, outputs $\pi_{h_{i^*(x)}}(x)$.
\end{definition}

The expected return of a white box ensemble {\bf h} of models is $\ecd[\pi_{\bf h}(x) \cdot y] = \sum_{i \in [k]} \ecd [\mathbbm{1}_{i = i^*(x)} \cdot \pi_{h_i}(x) \cdot y]$.
The expected self-evaluation of a white box ensemble is denoted $\ecd[\pi_{\bf h}(x) \cdot {\bf h}(x)] = \sum_{i \in [k]} \ecd [\mathbbm{1}_{i = i^*(x)} \cdot \pi_{h_i}(x) \cdot h_i(x)]$.

\subsection{Ensembling Models}
In this method, we will require consistency of our predictions with respect to another special collection of conditioning events: 
the sets on which each constituent model is most ``optimistic"---or has the highest self-assessed payoff.

\begin{definition}[Maximum Model Level Sets] 
    Fix a set of $k$ models $h^t_i: \cX \to \cY$ for $i \in [k]$. We refer to the maximum model level sets  as 
    \[ \calC^t_{[k]} = \left\{ \{ x | h^t_i(x) \cdot \pi_{h^t_i}(x) \geq h^t_j(x) \cdot \pi_{h^t_j}(x) \ \forall \ j \in [k] \} : i \in [k] \right\}. \]
    These correspond to the sets of examples on which each of the models $i$ has the highest self-assessed payoff. 
\end{definition}

The procedure for ensembling $k$ models involves modifying the constituent models so 
each is consistent with respect to the level sets of its own induced policy conditional on the identity of the model with the highest self-evaluated payoff. Informally, consistency on this set of events is useful because they are related to how the decision maker takes actions -- the ensemble follows the action of policy $i$ exactly when it has the highest self-assessed payoff. If each constituent model's predictions are self-consistent and consistent conditional on these sets characterizing when the decision maker is taking different actions, the resulting ensembling satisfies strong outcome guarantees.

\begin{algorithm} 
\caption{White Box Ensembling}
\label{alg-max-ensembling}
    \begin{algorithmic}
        \STATE Select consistency tolerance $\alpha$ and discretization $\cT$
        \STATE Initialize $t = 0$
        \STATE {\bf Input} Initial models $h^t_i, i \in [k]$
        \WHILE{ $\exists \ i$ s.t. $\exists \ C \in \calC^{\cT}_{\pi_{h^t_i}} \times \calC^t_{[k]}$ s.t. $\Pr[x \in C] \|\ecd[y - h^t_i(x) | x \in C]\|_{\infty} > \alpha$}
            \item $h^{t+1}_i = $ \textsc{Update}($h^t_i, \calC^{\cT}_{\pi_{h^t_i}} \times \calC^t_{[k]} $), for each $i \in [k]$
            \item $t = t+1$
        \ENDWHILE
        \STATE {\bf Output} white box ensemble ${\bf h} = (h^t_1, \ldots, h^t_k)$
    \end{algorithmic}
\end{algorithm}

\subsection{Convergence of White Box Ensembling}

Convergence of the white box ensembling method follows similarly to convergence of Algorithm \ref{alg}.
Algorithm \ref{alg-max-ensembling} repeatedly calls Algorithm \ref{alg} as a subroutine on each of the $k$ constituent models, on an adaptively chosen sequence of conditioning events.

\begin{lemma} \label{lem:update-conv-adaptive}
    Fix a model $h^1: \cX \to \cY$ and consistency tolerance $\alpha$.
    Let $h^t = \textsc{Update}(h^{t-1}, \calC^{t-1})$ for $t > 1$.
    For any sequence of, possibly adaptive, conditioning events $\calC^1, \ldots, \calC^t$, this process will terminate after at most $\frac{dM^2}{\alpha^2}$ rounds: that is, for $t > \frac{dM^2}{\alpha^2}$, it is the case that $h^t = \textsc{Update}(h^{t-1}, \calC^{t-1}).$
\end{lemma}

\begin{lemma} \label{lem:white-box-conv}
    Algorithm \ref{alg-max-ensembling} terminates within $k \cdot \frac{dM^2}{\alpha^2}$ rounds. 
\end{lemma}

The proofs of Lemma \ref{lem:update-conv-adaptive} and \ref{lem:white-box-conv} are deferred to Appendix \ref{ap:proofs}.

\subsection{Utility Guarantees}

We now analyze our white box ensembling method. First, we prove that it is self-consistent --- its self-assessed payoff is equal (up to error terms) to its realized payoff, in expectation. 

\begin{lemma} \label{lem:ensemble-consistent}
   Fix distribution $\cD$ and bucketing $\cT$, with $w = \sqrt{\alpha k / M}$.
   Let {\bf h} be an ensemble of $k$ models in which each $h_i$ is $\alpha-$consistent with respect to $\calC^{\cT}_{\pi_{h_i}} \times \calC_{[k]}$ for $i \in [k]$. 
   The expected self-evaluation of the ensemble {\bf h} is approximately equal to its expected revenue: $\ecd[\pi_{\bf h}(x) \cdot y] \geq \ecd[\pi_{\bf h}(x) \cdot {\bf h}(x)] - 2d \sqrt{\alpha k M}$.
\end{lemma}
The proof of Lemma \ref{lem:ensemble-consistent} is deferred to Appendix \ref{ap:proofs}.

We next prove two bounds on the performance of the method. The first states that---up to error terms--- the payoff of the ensemble is equal to the \emph{expected maximum} self assessed payoff of each of the constituent models. Notice that here we bound the performance of the ensemble by the expected max, which is larger than the max expectation, which corresponds to the best single constituent model. 

\begin{lemma} \label{lem:white-box-best}
Fix distribution $\cD$ and bucketing $\cT$, with $w = \sqrt{\alpha k / M}$.
Let {\bf h} be an ensemble of $k$ models in which $h_i$ is $\alpha-$consistent with respect to $\calC^{\cT}_{\pi_{h_i}} \times \calC_{[k]}$ for $i \in [k]$. 
Then,
\[ \ecd[\pi_{\bf h}(x) \cdot y] \geq \ecd[  \max_{j \in [k]} \pi_j(x) \cdot h_j(x)] - 2d \sqrt{\alpha k M}. \]
\end{lemma}

The next performance bound is a ``swap-regret'' like guarantee. It states that on the subset of examples on which the model chooses to follow policy $i$, it could not have improved by instead following some other policy $j$ --- simultaneously for all $i$ and $j$.  

\begin{lemma} \label{lem:white-box-swap}
    Fix distribution $\cD$ and bucketing $\cT$, with $w = \sqrt{\alpha k / M}$.
    Let {\bf h} be an ensemble of $k$ models in which $h_i$ is $\alpha-$consistent with respect to $\calC^{\cT}_{\pi_{h_i}} \times \calC_{[k]}$ for $i \in [k]$.
    Then, for all $\phi: [k] \to [k]$,
    
    \[ \ecd[\pi_{\bf h}(x) \cdot y] \geq \ecd[\pi_{h_{\phi(i^*(x))} } (x) \cdot y] - 4d \sqrt{\alpha k M}. \]
\end{lemma}

The proofs of both Lemmas \ref{lem:white-box-best} and \ref{lem:white-box-swap} are deferred to Appendix \ref{ap:proofs}.

\section{Black Box Ensembling Method}
\label{sec:single-model}

In this section, we describe the second, ``black box," method to ensemble models. 
This method involves maintaining a single, deterministic predictor which can be easily updated in the presence of new information, and requires only access to the induced policy of the models being ensembled.
Like the method described in Section \ref{sec:multiple-models}, the black box ensembling method enjoys transparent outcome guarantees. 
We show that this ensembling provides a ``swap style" utility guarantee, that the ensemble provably induces a payoff as high as any of the policies it is consistent to, conditioned on its action.

\paragraph{Interaction Model}
A decision maker has a predictive model $h$ and access to $k$ arbitrary policies $\pi_1, \ldots, \pi_k,$ whose information they want to incorporate into their predictive model.
The decision maker builds this ensemble by updating their model to satisfy consistency conditions relating to each policy $\pi_i$ which we describe below.
Since this procedure only involves updating a single set of predictions, the one that the decision maker begins with, they are able to ensemble policies generated arbitrarily - e.g. even ones without an underlying predictive model.

\subsection{Ensembling Policies}

This method involves maintaining a model that is unbiased with respect to the policy it itself induces, as well as each of the constituent policies to be ensembled. 

\begin{algorithm}[H]
\caption{Black Box Ensembling}
\label{alg-ensembling}
    \begin{algorithmic}
        \STATE {\bf Input} collection of $k$ policies $\{\pi_1, \ldots, \pi_k \}$
        \STATE Select consistency tolerance $\alpha$ and discretization $\cT$
        \STATE Initialize $t = 0$ and model $h^t$
        \WHILE{  $\exists \ C \in  \calC^{\cT}_{\pi_{h^t}} \cup_{i \in [k]} \calC^{\cT}_{\pi_i} $ s.t. $\Pr[x \in C] \|\ecd[y - h^t(x) | x \in C]\|_{\infty} > \alpha$ }
        \item $h^{t+1} = $ \textsc{Update}$(h^t, \calC^{\cT}_{\pi_{h^t}} \cup_{i \in [k]} \calC^{\cT}_{\pi_i} )$
        \item $t = t+1$
        \ENDWHILE
        \STATE {\bf Output} model $h^t$
    \end{algorithmic}
\end{algorithm}

\subsection{Utility Guarantees}

We now state and prove our main utility statement for our black box ensembling method: A swap-regret style guaranee that establishes that the policy induced by the ensemble model performs better than any of its constituent policies, not just overall, but also conditional on level sets of any policy in its ensemble. 

\begin{lemma} \label{lem:black-box-utility}
    Fix distribution $\cD$ and bucketing $\cT$.
    Let $h$ be an $\alpha-$self-consistent model that is $\alpha-$consistent with respect to a collection of policies $\cP$.
    Then, $\mathbb{E}_{\cD}[\pi_h(x) \cdot y | x \in C] \geq \mathbb{E}_{\cD} [ \pi(x) \cdot y | x \in C]- \frac{2 \alpha d}{\Pr[\pi(x)_i \in \tau] } - 2wMd ,$ for all $\pi, \pi' \in \cP \cup \{\pi_h\}$ and all $C \in \calC^{\cT}_{\pi'}$.
\end{lemma} 

The proof of Lemma \ref{lem:black-box-utility} is deferred to Appendix \ref{ap:proofs}.

\section{Experimental Evaluation} \label{sec:experimental}

We empirically evaluate both of our ensembling methods across four experimental setups with synthetic data. The two main results of our experimental section are as follows: 
\begin{itemize}
    \item Both ensembling methods yield decisions which improve on the outcomes from the initial constituent policies.
    \item The white box ensembling method tends to yield better outcomes than the black box ensembling method, at the expense of higher runtime in practice.
\end{itemize}

\begin{table}[H]
\begin{center}
\begin{tabular}{c | c c }
Experiment & Initial model specialization & Optimization Problem \\
\hline\hline
 \textbf{A} & Coordinate-wise & Covariance-constrained\\ 
 \textbf{B} & Group-wise & Covariance-constrained \\  
 \textbf{C} & Coordinate-wise & Linearly-constrained \\
 \textbf{D} & Group-wise & Linearly-constrained
\end{tabular}
\caption{Experiment Design}
\label{table:exp}
\end{center}
\end{table}

\subsection{Design Details}

Four different experiments to compare the white box and black box algorithms' performance were run on synthetic data. In each, the same data were used, but the set of initial models ingested by the algorithms and the downstream optimization problem their policies used differed. We list the different experiment's settings in Table \ref{table:exp}, and detail the specifics of their model design and optimization problems below. 

\paragraph{Synthetic Data Generation:} Experiments were run on synthetic data with 20 features and 4-dimensional labels. The feature data are multivariate normally distributed with a randomly generated covariance matrix, with a final categorical feature which uniformly randomly assigned datapoints to one of 5 categories. The 4-dimensional labels had a noisy linear relationship with the features. 

\paragraph{Initial Models:} In evaluating our two algorithms, we modify the types of initial models ensembled. By set of initial models, we refer to the set of models $h_1^0, \ldots, h_k^0$ input to the white box Algorithm \ref{alg-max-ensembling}. The black box Algorithm \ref{alg-ensembling} takes policies $\pi_1, \ldots, \pi_k$ as input, which we take to be the policies induced by $h_1, \ldots, h_k$, i.e. $\pi_{h_1}(x), \ldots, \pi_{h_k}(x)$ where $\pi_h(x) = \arg\max_{a \in \Omega} a \cdot h(x)$ subject to the particular problem's constraints. The black box algorithm also requires an initial model $h^0$, which we initialize to be a naive model which predicts the label mean for all points. 

Ensembling is most interesting if the constituent models have different strengths, as they might if for example they were generated with differing data access. Here, initial models were all gradient boosted regression trees of depth 6 with learning rate 0.1, but were constructed to specialize in different ways. Models which specialized \textit{coordinate-wise} were each trained on only a single coordinate of the multivariate regression problem: on the coordinate the model specialized in, it predicted according to its gradient boosting regressor, and in the other coordinates it predicted the label mean, as shown in Equation \ref{eq:coord}:

\begin{equation}
h_i(x) = (\bar{y}_0, \ldots, \texttt{GradientBoostRegressor}(x), \ldots, \bar{y}_d).
\label{eq:coord}
\end{equation}

In experiments A and C, four such initial models were formed, one for each coordinate of the label. 

In experiments B and D, models specialized \textit{group-wise}. As discussed above, the features of the data included a categorical variable, which we then used to subset the data for these models. Note this might correspond in a real-world setting to generating a series of models where each had access to different subpopulations. Each of the five models generated specialized on a different subpopulation: as shown in Equation \ref{eq:group}, for that subpopulation, it predicted using a gradient boosted regression (one trained per coordinate, as standard gradient boosting regression tree implementations do not handle multivariate labels), and otherwise predicted the label mean.

\begin{equation}
h_g(x) = 
\begin{cases}
    (\texttt{GradientBoostRegressor}(x)_1, \ldots, \texttt{GradientBoostRegressor}(x)_d) & \text{if} \quad  x \in g, \\
    \bar{\mathbf{y}} & \text{else.}
\end{cases}
\label{eq:group}
\end{equation}

For experiments B and D, there were 5 such initial models generated, one for each subgroup.

For all experiments, the initial training was done using 10,000 datapoints drawn from the synthetic distribution.

\paragraph{Downstream Optimization Problem:} We considered two different forms of optimization problem: one which constrained the policies by the covariance of the labels, similar to mean-variance optimization in portfolio design, and a second which was a set of simple linear constraints. The covariance constrained optimization is shown here in Figure \ref{fig:cov-const}. The linearly constrained optimization, which simply bounds the amount of weight on sets of coordinates, is shown in Figure \ref{fig:lin-const}.

\begin{figure}[h]
    \centering
    \begin{minipage}{0.9\textwidth}
        \begin{align*}
            &\text{maximize} && \boldsymbol{\pi} \cdot \mathbf{y} \\
            &\text{subject to} && \sum_{j=1}^n \pi_j = 1, \\
            & && 0 \leq \pi_j \leq 1 \quad \forall j \in \{1, 2, \ldots, n\},\\
            & && \boldsymbol{\pi}^T \mathbf{C} \boldsymbol{\pi} \leq \alpha
        \end{align*}
    \end{minipage}
    \caption{Covariance-constrained policy. Here, $\boldsymbol{\pi}$ is the vector of decision variables (i.e. the induced policy), where each term is constrained to be a probability. In the final constraint, $\mathbf{C}$ is the covariance matrix of the true labels, and the quadratic constraint is bounded by a value $\alpha$ which was chosen appropriately to the scale of $\mathbf{C}$.}
    \label{fig:cov-const}
\end{figure}

\begin{figure}[h]
    \centering
    \begin{minipage}{0.9\textwidth}
        \begin{align*}
            &\text{maximize} && \boldsymbol{\pi} \cdot \mathbf{y} \\
            & \text{subject to} && 0 \leq \pi_j \leq 1 \quad \forall j \in \{1, 2, \ldots, n\},\\
            & && \pi_0 + \pi_1 \le \alpha \\
            & && \pi_1 + \pi_2 \le \beta \\
        \end{align*}
    \end{minipage}
    \caption{Linearly-constrained policy. The constraints $\alpha$ and $\beta$ were chosen to be 0.5 and 0.6 respectively for the purposes of our experiments.}
    \label{fig:lin-const}
\end{figure}

\paragraph{Debiasing Algorithm Implementation:} For all experiments, the debiasing algorithms were run on a subset of 400 synthetically generated datapoints. This data split was primarily chosen due to the expense of running on a larger sample, as optimizations have to be run sample-by-sample. All experiments were run on a local 32-core machine, but without any parallelization to leverage the multiple cores. Gurobi was used to solve the optimizations, and the initial models were trained using scikit-learn (\cite{gurobi, scikit-learn}). 

\paragraph{Source Code:} The specific data generation process and all source code may be found at \url{https://github.com/globusharris/ensembling-constrained-optimization}

\subsection{Evaluation of Performance} In Figure \ref{fig:main-exp}, we examine the performance of the two algorithms across the four experimental setups. In all experiments, we see that both algorithms have their predicted payoff converge to their realized payoffs, and that both achieve higher payoff than the highest payoff realized by the policies induced by the initial models. As discussed in Section \ref{sec:multiple-models}, the theory for the white box algorithm does not guarantee that the final ensemble outperform the initial models, but in practice we empirically see that it  does, and by substantial margins. 

In particular, in experiments A, B, and C, the white box algorithm converges to the \textit{best possible} payoff. In experiment D it very nearly does. The black box method, however, does not always converge to the best possible model; in particular in experiment A it only marginally improves compared to the initial models, while the white box model converges to optimal.

\begin{figure}[H]
    \centering
    \includegraphics[width=\textwidth]{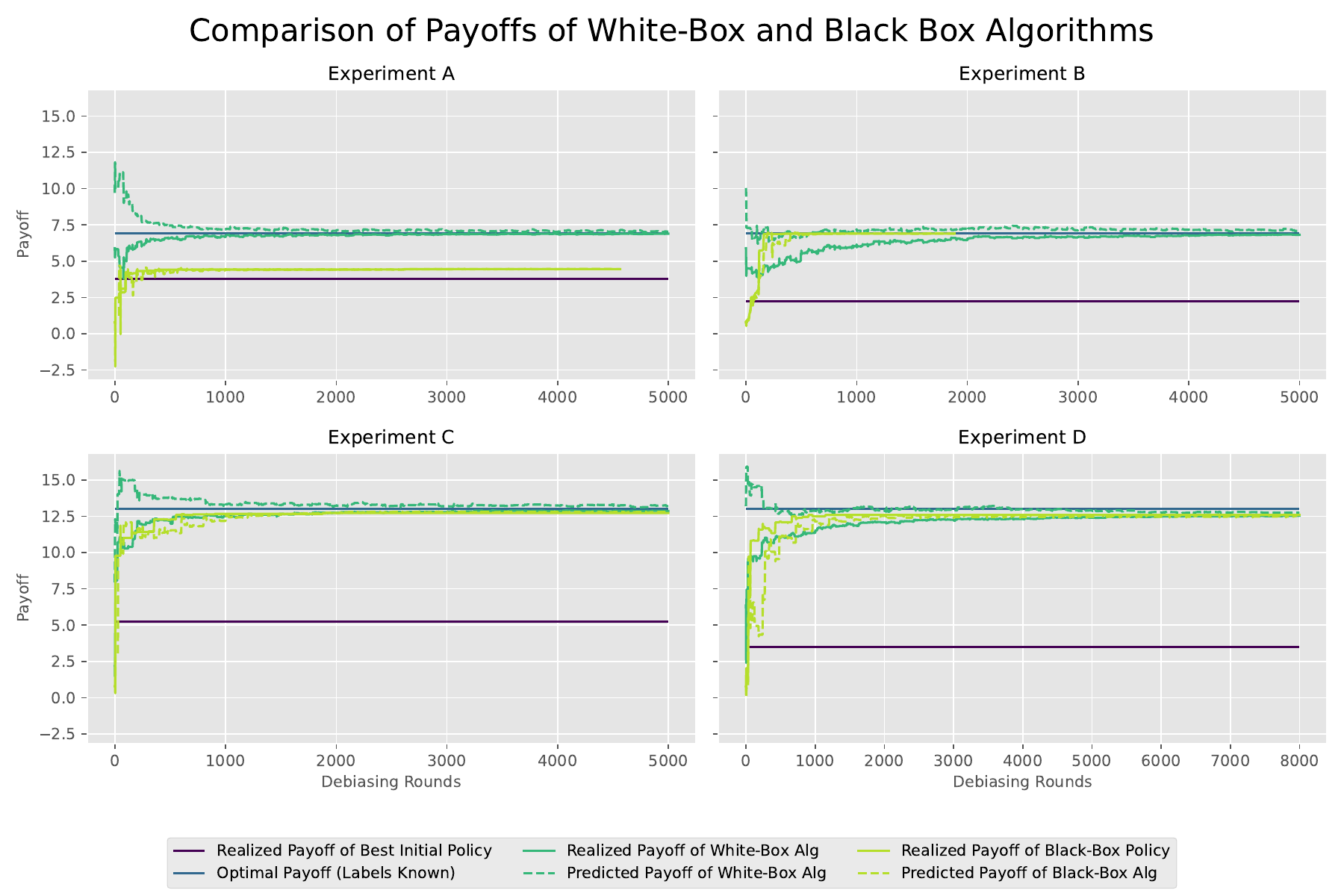}
    \caption{Comparison of payoffs of white box and black box algorithms in each experiment. Dashed lines correspond to predicted payoff, while solid green and teal are realized payoff.}
    \label{fig:main-exp}
\end{figure}

\paragraph{Behavior of Component Models for White Box Ensemble} Not captured in Figure \ref{fig:main-exp} is the dynamics of the white box ensembling method. Are all of the models actually updated during the debiasing procedure, and does the final ensemble actually select amongst multiple policies? As we see in Figure \ref{fig:wb-component}, in our experiments all of the constituent models are significantly updated in the ensembling process, leading to final policies that in expectation lead to roughly the same payoff as each other (all significantly improved over the initial policies). 

This shows that all of the constituent policies are updated to have roughly the same expected payoff. But how similar are their actions round by round? We see in Figure \ref{fig:policy-choice} that in experiments A and B, the ensemble selects amongst all of its constituent policies with roughly equal proportion, while in experiments C and D, it primarily chooses the first policy. However, for all experiments, we see in Figure \ref{fig:policy-var} that the constituent policies do not substantially vary in their choices after the initial rounds of debiasing, so the fact that experiments C and D choose policy 0 more often than the others in the final rounds of debiasing does not materially mean that its outcome is substantially different than if it were to ensemble using a different procedure. This is consistent with the use of similar techniques by \cite{stevenetal} for the explicit purpose of resolving disagreements in the actions induced by predictive models. 

\begin{figure}[H]
    \includegraphics[width=\textwidth]{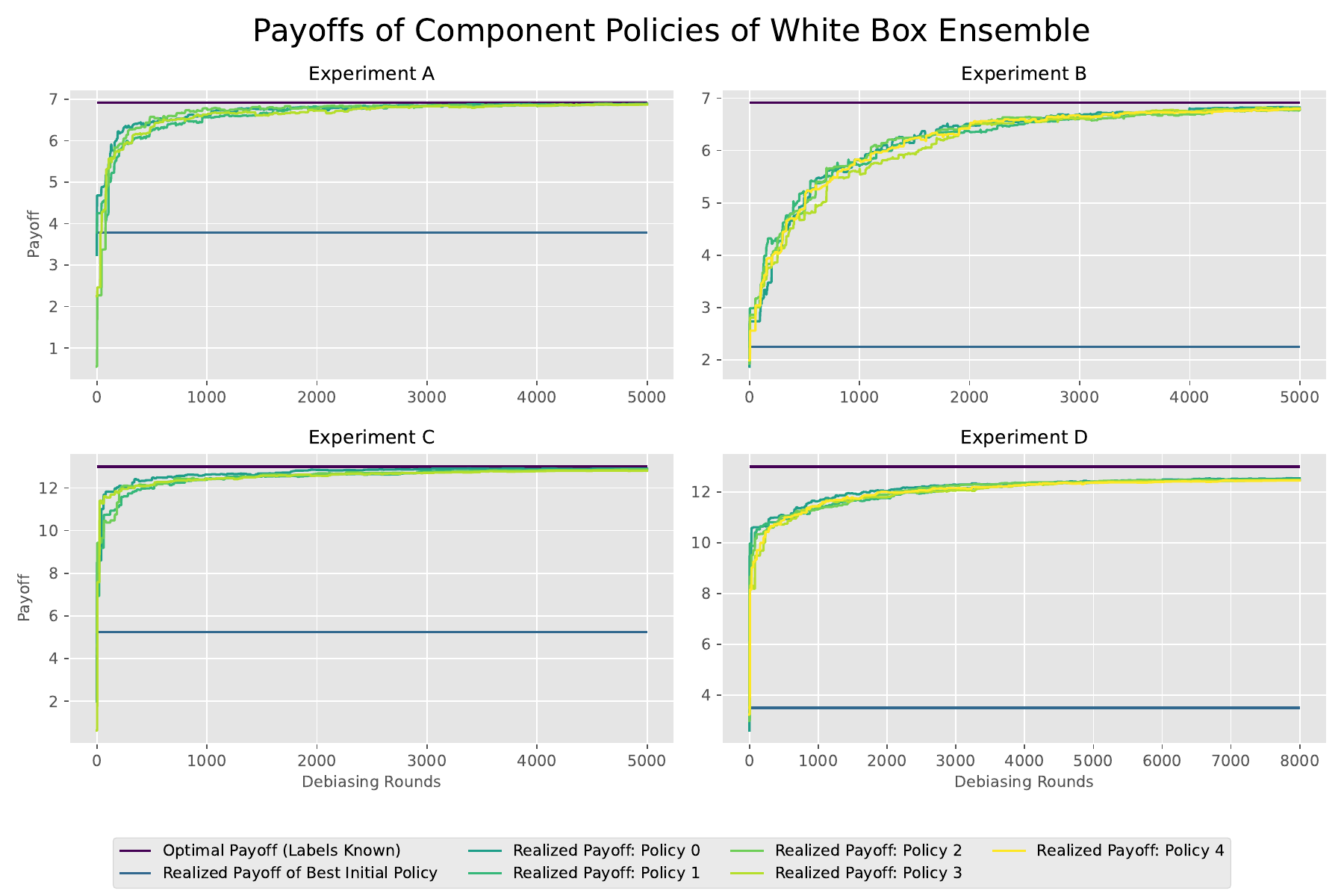}
    \caption{Realized payoffs of the component policies of the white box ensemble over round of debiaising. Note that experiments A and C have 4 initial models and accompanying induced policies (one specializing in each coordinate of the prediction) while experiments B and D have 5 initial models (one specializing in each subgroup of the dataset).}
    \label{fig:wb-component}
\end{figure}

\begin{figure}[H]
    \begin{centering}
    \includegraphics[width=0.9\textwidth]{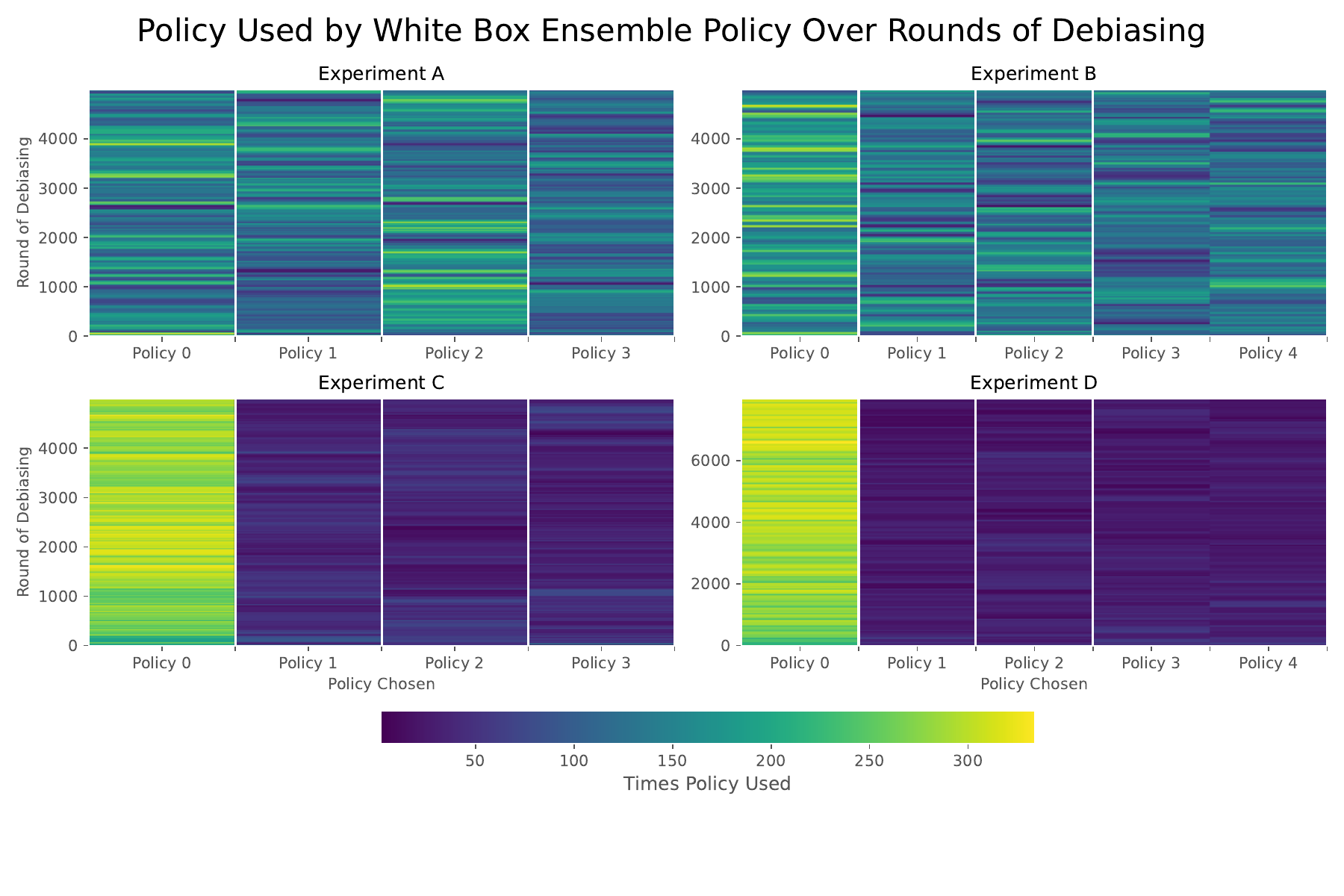}
    \end{centering}
    \caption{Heatmap of distribution of how constituent policies were chosen by the ensembling process across rounds of debiasing. For instance, in experiments A and B, no single policy is prioritized substantially more than the others, while in experiments C and D, the first policy ends up being prioritized.}
    \label{fig:policy-choice}
\end{figure}

\begin{figure}[H]
    \includegraphics[width=\textwidth]{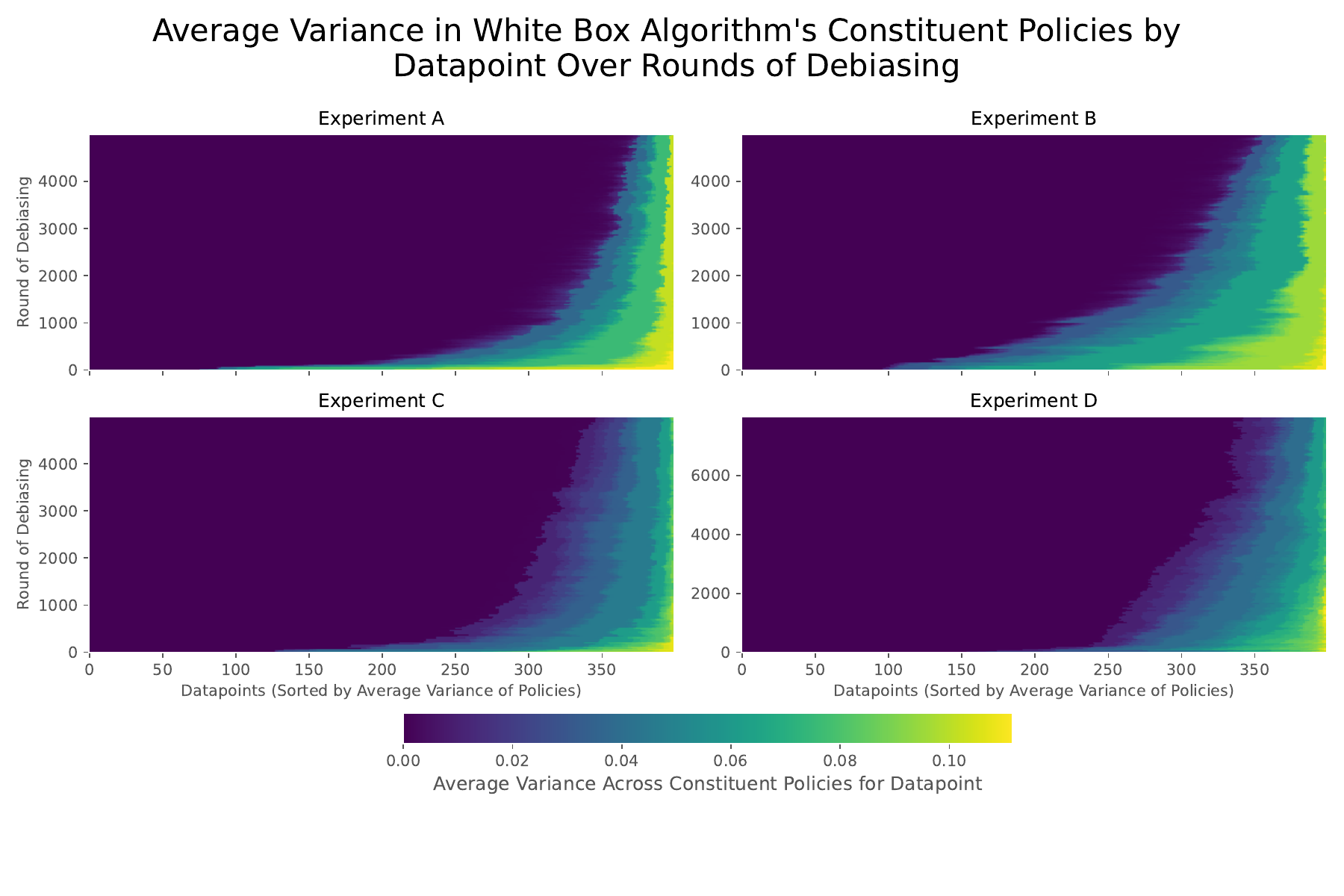}
    \caption{Heatmap of the average variance across the ensembled policies for each training datapoint over the rounds of debiasing. The variance is taken over each coordinate of the induced policies, then averaged across coordinate, for each datapoint at each round of debiasing. For ease of interpretation, on the $x$-axis, datapoints in each round of debiasing are sorted by their average variance.}
    \label{fig:policy-var}
\end{figure}

\paragraph{Efficiency} Since the white box algorithm in our experiments converges to near optimal payoff while the black box algorithm does not, it is tempting to advocate for using the white box algorithm whenever white box access is possible. However, it is substantially slower, both in terms of the time it takes to converge and in practical runtime. In practice, the algorithms' bottleneck is the evaluation of the downstream optimization problem. The black box algorithm requires far fewer of these computations, which makes a substantial difference in practice: for instance, on our machine, the black box algorithm for Experiment A took 3.8 minutes, while the white box algorithm took 83 minutes.

\section{Discussion of Limitations}
Our algorithms operate in the standard batch/distributional setting, and  the guarantees that we prove are limited to the setting in which the data encountered at deployment time is distributed identically to the data used in training. An important open question is how to adapt similar practical techniques to give more robust techniques that allow for various kinds of distribution shift. Techniques of \cite{noarov2023highdimensional} could be adapted to give similar ensembling algorithms in the online adversarial setting when the prediction target can be observed shortly after prediction at test time; but these algorithms are mainly of theoretical interest, and more practical variants would need to be developed. Even in the batch setting, our white box ensembling method is computationally very expensive, and significantly more efficient algorithms with similar guarantees would be important improvements. 

\bibliographystyle{plainnat}
\bibliography{bib}

\newpage
\appendix
\section{Proofs}
\label{ap:proofs}

\begin{proof}[Proof of Lemma \ref{lem:update-conv}]
    Each round within the procedure \textsc{Update($\cdot, \cdot)$} finds a consistency violation with respect to the input model $h$ and collection of sets $\calC$ of the form: 
    $\Pr[x \in C] \|\E_{\cD} [ y - h(x) | x \in C \|_{\infty} > \alpha$ for some $C \in \calC$. 
    By Lemma \ref{lem:monotone-sq-error}, we know that the decrease of squared error between the policies of two adjacent rounds of $\textsc{Update}(\cdot, \cdot)$ denoted $h$ and $h'$ satisfy:
    \begin{align*}
         \ecd[\|h(x) - y\|_2^2] - \ecd[\|h'(x) - y\|_2^2] = \Pr[x \in C] \cdot \| \ecd[y - h(x) | x \in C ]\|_2^2.
    \end{align*}
    Additionally, we know that:
    \begin{align*}
       \Vert\E_{\cD}[y - h(x) | x \in C ]\Vert_2^2 \geq \Vert\E_{\cD}[y - h(x) | x \in C ]\Vert_{\infty}^2.
    \end{align*}
    By the stopping condition of $\textsc{Update}(\cdot, \cdot)$, we know that while the procedure has not terminated:
    \begin{align*}
        \max_{C \in \calC} \Pr[x \in C] \Vert\E_{\cD}[y - h(x) | x \in C ]\Vert_{\infty} > \alpha.
    \end{align*}
    Therefore, we know that in each iteration of $\textsc{Update}(\cdot, \cdot)$ the squared error of the model $h$ drops by
    \begin{align*}
         \ecd[\|h(x) - y\|_2^2] - \ecd[\|h'(x) - y\|_2^2] \geq \frac{\alpha^2}{\Pr[x \in C]} \geq \alpha^2.
    \end{align*}
    Since each $h(x) \in [0, M]^d \ \forall \ x \in \cX$ and $\cY \subseteq [0, M]^d$, we know that the squared error can be at most $d \cdot M^2$. 
\end{proof}

\begin{proof} [ Proof of Lemma \ref{lem:single-consistent} ] 
    \begin{align*}
        \E_{\cD}[\pi_h(x) \cdot y] &\geq \E_{\cD}[\pi_h(x) \cdot h(x)] - 2d \sqrt{\alpha M} \\
        & \geq \ecd[\pi(x) \cdot h(x)] - 2d \sqrt{\alpha M} \\
        & = \ecd[\sum_{i \in [d]} \pi(x)_i \cdot h(x)_i ] - 2d \sqrt{\alpha M} \\
        & \geq \sum_{i \in [d]} \sum_{\tau \in \cT} \Pr[ \pi(x)_i \in \tau] \left( \ecd[\tau_1 \cdot h(x)_i - |\tau_1 - \pi(x)_i| \cdot|h(x)_i| \ | \pi(x)_i \in \tau]  \right) -  2d \sqrt{\alpha M} \\
        & \geq \sum_{i \in [d]} \sum_{\tau \in \cT} \Pr[ \pi(x)_i \in \tau] \left( \ecd[\tau_1 \cdot h(x)_i | \pi(x)_i \in \tau]  - \frac{wM}{2} \right) -  2d \sqrt{\alpha M} \\
        & \geq \sum_{i \in [d]} \sum_{\tau \in \cT} \Pr [\pi(x)_i \in \tau] \left( \ecd[\tau_1 \cdot y_i | \pi(x)_i \in \tau] - \frac{\alpha}{\Pr[\pi(x)_i \in \tau]} - \frac{wM}{2} \right) -  2d \sqrt{\alpha M} \\
        & \geq \sum_{i \in [d]} \sum_{\tau \in \cT} \Pr [\pi(x)_i \in \tau] \left( \ecd[\pi(x)_i \cdot y_i - |\tau_1 - \pi(x)_i| \cdot |y_i| \ | \pi(x)_i \in \tau] - \frac{\alpha}{\Pr[\pi(x)_i \in \tau]} - \frac{wM}{2} \right) \\
        & \quad \quad \quad \quad \quad -  2d \sqrt{\alpha M} \\
        & \geq \sum_{i \in [d]} \sum_{\tau \in \cT} \Pr [\pi(x)_i \in \tau] \left( \ecd[\pi(x)_i \cdot y_i | \pi(x)_i \in \tau] - \frac{\alpha}{\Pr[\pi(x)_i \in \tau]} - wM \right) -  2d \sqrt{\alpha M}, \\
        & = \sum_{i \in [d]} \sum_{\tau \in \cT} \Pr [\pi(x)_i \in \tau] \ecd[\pi(x)_i \cdot y_i | \pi(x)_i \in \tau] - \frac{\alpha d }{w} - wMd -  2d \sqrt{\alpha M},
\end{align*}
    where the first inequality holds by Corollary \ref{cor:self-rev}, the second by the optimality of policy $\pi_h$ with respect to $h$ rather than the policy $\pi$, and the fifth by the consistency of $h$ with respect to the policy $\pi$. 
    Then, taking $ w = \sqrt{\alpha / M},$ we have
    \begin{align*}
        \sum_{i \in [d]} \sum_{\tau \in \cT} \Pr [\pi(x)_i \in \tau] \ecd[\pi(x)_i &\cdot y_i | \pi(x)_i \in \tau] - \frac{\alpha d }{w} - wMd -  2d \sqrt{\alpha M} \\
        = \ecd[\pi(x) \cdot r] - 4 d \sqrt{\alpha M}.
    \end{align*}
\end{proof}

\begin{proof} [Proof of Lemma \ref{lem:update-conv-adaptive}]
    This proof follows similarly to that of Lemma \ref{lem:update-conv}.
    As stated in Lemmas \ref{lem:monotone-sq-error} and \ref{lem:update-conv}, we know that between two adjacent rounds within the $\textsc{Update}(f, \cdot)$ procedure for some model $f$, the squared error of the model $f$ drops by at least $\alpha^2$.
    Therefore, we know that between two adjacent invocations of $\textsc{Update}(h^t, \cdot)$ and $\textsc{Update}(h^{t+1}, \cdot)$ within Algorithm \ref{alg-max-ensembling}, the squared error of model $h^{t+1}$ must be at least $\alpha^2$ less than the squared error of model $h^t$. Since $\cY \subseteq \mathbb{R}^d$ and $\|y\|_{\infty} \leq M$ for all $y \in \cY$, it must be the case that Algorithm \ref{alg-max-ensembling} terminates after at most $\frac{dM^2}{\alpha^2}$ invocations of Algorithm \ref{alg}.
\end{proof}

\begin{proof} [Proof of Lemma \ref{lem:white-box-conv}]
    This follows directly from $k$ applications of Lemma \ref{lem:update-conv-adaptive}, since Algorithm \ref{alg-max-ensembling} maintains $k$ separate predictors which are being iteratively updated using the $\textsc{Update}(\cdot, \cdot)$ procedure on an adaptively chosen sequence of conditioning events.
\end{proof}

\begin{proof}[Proof of Lemma \ref{lem:ensemble-consistent}]
    \begin{align*}
        \ecd[\pi_{\bf h} (x) \cdot y] &= \sum_{x \in \cX} \Pr[X = x] \left( \pi_{\bf h}(x) \cdot \E[r|x] \right) \\
        &=  \sum_{x \in \cX}  \sum_{i \in [k]} \Pr[X = x] \1_{i = i*(x)} \cdot \pi_{h_i}(x) \cdot \E[y|x] \\
        &\geq \sum_{i \in [k]} \sum_{\tau \in \cT} \sum_{j \in [d]}  \Pr[ \pi_{h_i}(x)_j \in \tau, i^*(x) = i] \left( \ecd[ \tau_1 \cdot y_j - | \tau_1 - \pi_{h_i}(x)_j | \cdot |y_j| \ | \pi_{h_i}(x)_j \in \tau, i^*(x) = i] \right) \\
        &\geq \sum_{i \in [k]} \sum_{\tau \in \cT} \sum_{j \in [d]}  \Pr[ \pi_{h_i}(x)_j \in \tau, i^*(x) = i] \left( \ecd[ \tau_1 \cdot y_j | \pi_{h_i}(x)_j \in \tau, i^*(x) = i] - \frac{wM}{2} \right) \\
        &\geq \sum_{i \in [k]} \sum_{\tau \in \cT} \sum_{j \in [d]}  \Pr[ \pi_{h_i}(x)_j \in \tau, i^*(x) = i] \big( \ecd[ \tau_1 \cdot h_i(x)_j| \pi_{h_i}(x)_j \in \tau, i^*(x) = i]  \\
        & \quad \quad \quad \quad  - \frac{\alpha}{\Pr[\pi_{h_i}(x)_j \in \tau, i^*(x) = i]} - \frac{wM}{2} \big) \\
        &\geq \sum_{i \in [k]} \sum_{\tau \in \cT} \sum_{j \in [d]}  \Pr[ \pi_{h_i}(x)_j \in \tau, i^*(x) = i] \\
        & \quad \quad \quad \quad \big( \ecd[ \pi_{h_i}(x)_j \cdot h_i(x)_j - |\tau_1 - \pi_{h_i}(x)_j| \cdot |h_i(x)_j| \ | \pi_{h_i}(x)_j \in \tau, i^*(x) = i] \\
        & \quad \quad \quad \quad - \frac{\alpha}{\Pr[\pi_{h_i}(x)_j \in \tau, i^*(x) = i]} - \frac{wM}{2} \big) \\
        &\geq \sum_{i \in [k]} \sum_{\tau \in \cT} \sum_{j \in [d]}  \Pr[ \pi_{h_i}(x)_j \in \tau, i^*(x) = i] \big( \ecd[ \pi_{h_i}(x)_j \cdot h_i(x)_j| \pi_{h_i}(x)_j \in \tau, i^*(x) = i] \\
        & \quad \quad \quad \quad - \frac{\alpha}{\Pr[\pi_{h_i}(x)_j \in \tau, i^*(x) = i]} - wM \big) \\
        &= \ecd[\pi_{\bf h}(x) \cdot {\bf h}] - \alpha k |\cT| d - wMd \\
        &= \ecd[\pi_{\bf h}(x) \cdot {\bf h}] - \frac{\alpha k d}{w} - wMd,
    \end{align*}
    where the third inequality follows from the consistency conditions on the ensemble {\bf h}.
    Then, setting $w = \sqrt{\alpha k / M},$ we have
    \begin{align*}
    \ecd[\pi_{\bf h}(x) \cdot {\bf h}] - \frac{\alpha k d}{w} - wMd = \ecd[\pi_{\bf h}(x) \cdot {\bf h}] - 2d \sqrt{\alpha k M}.
    \end{align*}
\end{proof}

\begin{proof}[Proof of Lemma \ref{lem:white-box-best}]
    \begin{align*}
        \ecd[\pi_{\bf h}(x) \cdot y] &\geq \ecd[ \pi_{\bf h}(x) \cdot {\bf h}(x) ] - 2d \sqrt{\alpha k M} \\ 
        & \geq \ecd[ \max_{j \in [k]} \pi_{h_j}(x) \cdot h_j(x)] - 2d \sqrt{\alpha k M} ,
    \end{align*}
    where the first inequality follows from Lemma \ref{lem:ensemble-consistent} and the second inequality follows from Definition \ref{def:ensemble-policy}.
\end{proof}

\begin{proof} [Proof of Lemma \ref{lem:white-box-swap}]
    \begin{align*}
        \ecd[\pi_{\bf h}(x) \cdot y] &\geq \ecd[\pi_{\bf h}(x) \cdot {\bf h}(x)] - 2d \sqrt{\alpha k M} \\
        & \geq \ecd[\pi_{h_{\phi(i^*(x))} } (x) \cdot h_{\phi(i^*(x))} (x)] - 2d \sqrt{\alpha k M} \\
        & \geq \ecd[\pi_{h_{\phi(i^*(x))}} (x) \cdot y] - 4 d \sqrt{\alpha k M}, 
    \end{align*}
    where the first inequality follows from Lemma \ref{lem:ensemble-consistent}, the second inequality follows from Definition \ref{def:ensemble-policy}, and the third inequality follows from the ``cross" consistency conditions, that model $h_s$ is approximately consistent with respect to its own policy conditioned on model $h_t$ having the highest self-evaluation, for all $s, t \in [k]$.
\end{proof}

\begin{proof}  [Proof of Lemma \ref{lem:black-box-utility}]
    Fix $\pi \in \cP \cup \{ \pi_h \}, i \in [d], \tau \in \cT$.
    \begin{align*}
        \mathbb{E}_{\cD}[\pi_h(x) \cdot r | \pi(x)_i \in \tau] &=  \mathbb{E}_{\cD}[ \sum_{j \in [d]} \pi_h(x)_j \cdot y_j | \pi(x)_i \in \tau ] \\ 
        &= \sum_{j \in [d]} \mathbb{E}_{\cD}[\pi_h(x)_j \cdot y_j |\pi(x)_i \in \tau ] \\
        &\geq \sum_{j \in [d]} \mathbb{E}_{\cD}[\tau_1 \cdot y_j - | \tau_1 - \pi_h(x)_j | \cdot | y_j | \ |\pi(x)_i \in \tau ] \\
        &\geq \sum_{j \in [d]} \left( \mathbb{E}_{\cD}[\tau_1 \cdot y_j |\pi(x)_i \in \tau ] - \frac{wM}{2} \right) \\
        &\geq \sum_{j \in [d]} \mathbb{E}_{\cD}[\tau_1 \cdot h(x) |\pi(x)_i \in \tau ] - \frac{wMd}{2} - \frac{\alpha d}{\Pr[ \pi(x)_i \in \tau]} \\
        &\geq \mathbb{E}_{\cD}[\pi_h(x) \cdot h(x) | \pi(x)_i \in \tau ] - wMd - \frac{\alpha d}{\Pr[\pi(x)_i \in \tau] } \\
        &\geq \mathbb{E}_{\cD}[\pi(x) \cdot h(x) | \pi(x)_i \in \tau ] - wMd - \frac{\alpha d}{\Pr[\pi(x)_i \in \tau] } \\
        &\geq \mathbb{E}_{\cD}[\tau_1 \cdot h(x) | \pi(x)_i \in \tau ] - \frac{3wMd}{2} - \frac{ \alpha d}{\Pr[\pi(x)_i \in \tau] } \\
        &\geq \mathbb{E}_{\cD}[\tau_1 \cdot y | \pi(x)_i \in \tau ] - \frac{3wMd}{2}  - \frac{2 \alpha d}{\Pr[\pi(x)_i \in \tau] } \\
        &\geq \mathbb{E}_{\cD}[\pi(x) \cdot y | \pi(x)_i \in \tau ] - 2wMd - \frac{2 \alpha d}{\Pr[\pi(x)_i \in \tau] }
    \end{align*}
    where the third and penultimte inequalities follow from the consistency of $h$ to policy $\pi$ and the fifth from the pointwise optimality of policy $\pi_h$ to model $h$.
\end{proof}

\end{document}